\documentclass[10pt]{article}

\usepackage[preprint]{tmlr}

\usepackage{amsmath,amsfonts,bm}

\def\eqref#1{equation~\ref{#1}}

\def\1{\bm{1}}

\DeclareMathAlphabet{\mathsfit}{\encodingdefault}{\sfdefault}{m}{sl}
\SetMathAlphabet{\mathsfit}{bold}{\encodingdefault}{\sfdefault}{bx}{n}

\usepackage{amsmath}
\usepackage{amssymb}
\usepackage{mathtools}
\usepackage{graphicx}
\usepackage{booktabs}
\usepackage{multirow}
\usepackage{subcaption}
\usepackage{xcolor}
\usepackage{hyperref}
\usepackage{url}
\usepackage{amsthm}

\newtheorem{theorem}{Theorem}[section]

\title{RAMP: Recurrent Adaptive Mixed Precision for Efficient Prefill in Recurrent-Depth Language Models}

\author{
\name Ke Wan \\
\addr Department of Computer Science, University of Virginia \\
\email{tbn5pj@virginia.edu}
\AND
\name Chen Chen \\
\addr Department of Computer Science, University of Virginia \\
\email{zrh6du@virginia.edu}
}

\hypersetup{
  pdftitle={RAMP: Recurrent Adaptive Mixed Precision for Efficient Prefill in Recurrent-Depth Language Models},
  pdfauthor={Ke Wan and Chen Chen}
}

\begin{document}

\maketitle
\fancyhead{}
\renewcommand{\headrulewidth}{0pt}

\begin{abstract}
Recurrent-depth language models improve inference-time reasoning by repeatedly applying a shared network to refine internal representations, but the additional recurrent computation also increases inference cost. This overhead is particularly pronounced during long-context prefill, where the full input sequence is repeatedly processed across recurrent depth. Weight-and-activation quantization offers a natural way to accelerate this computation through lower-precision arithmetic, yet repeatedly applying quantized computation introduces numerical perturbations that accumulate along the recurrent trajectory. We systematically analyze how such errors propagate across recurrent depth and show that quantization errors introduced earlier are attenuated by more subsequent recurrent steps and therefore have a smaller effect on the final representation. Matched-budget experiments strongly support this prediction, showing that placing quantized computation earlier in recurrent depth consistently yields the best fidelity. Based on this insight, we introduce \texttt{RAMP} (\textbf{R}ecurrent \textbf{A}daptive \textbf{M}ixed \textbf{P}recision), which performs early recurrent refinement using low-precision weight-and-activation computation and restores high precision for terminal refinement. Across context lengths up to 4K, \texttt{RAMP} preserves near-full-precision prefill quality while achieving up to \textbf{1.16$\times$} end-to-end first-prefill speedup on an NVIDIA H100, corresponding to a \textbf{13.9\%} latency reduction.
\end{abstract}

\section{Introduction}
\label{sec:intro}

Recurrent-depth language models provide a new way to scale inference-time computation without generating additional intermediate tokens: a shared transformation is repeatedly applied to refine internal representations \citep{wan2026attention}. 
Increasing recurrent depth has been shown to improve model performance \citep{geiping2026scaling}, suggesting that additional recurrent computation can translate into stronger reasoning.
However, this benefit comes with increased inference latency.
The cost is especially pronounced during long-context prefill, where each recurrent step repeatedly processes the full input sequence.
This motivates a fundamental question: \emph{can we accelerate recurrent-depth prefill while preserving the performance gains of deeper recurrent computation?}


Quantization provides a natural way to reduce inference cost
\citep{dettmers2022gpt3,xiao2023smoothquant,lin2024awq}.
In particular, weight-and-activation quantization lowers the precision of both model parameters and intermediate activations, enabling faster low-precision arithmetic during prefill
\citep{xiao2023smoothquant,zhao2024atom,lin2025qserve}.
A straightforward approach for recurrent-depth models is therefore to quantize the recurrent computation uniformly throughout prefill. However, repeatedly applying low-precision computation to an evolving representation introduces numerical perturbations at every recurrent step, causing the quantized trajectory to deviate progressively from its full-precision counterpart. A natural alternative is to use mixed precision across recurrent depth, assigning quantized computation to some recurrent steps while preserving the original precision for others. This raises a fundamental question: \emph{how should a fixed precision budget be allocated across recurrent depth to maximize efficiency while preserving the fidelity of full-precision recurrence?}


To understand this tradeoff, we analyze how quantization error propagates through recurrent computation. Our analysis reveals a depth asymmetry in precision requirements: under locally contractive recurrent dynamics, perturbations introduced earlier are attenuated by more subsequent recurrent steps, whereas perturbations introduced near the end retain a larger effect on the final representation. Consequently, under a fixed precision budget, placing low-precision computation early and preserving high precision for later recurrent steps minimizes the theoretical final-error bound. Based on this insight, we introduce \texttt{RAMP} (\textbf{R}ecurrent \textbf{A}daptive \textbf{M}ixed \textbf{P}recision), which performs early recurrent refinement with low-precision computation and reserves high precision for terminal refinement.
Our contributions are summarized as follows:
\begin{itemize}
    \item We theoretically show that, under a fixed precision budget, quantized computation should be placed earlier in recurrent depth while high-precision computation should be preserved for later steps. As shown in Figure~\ref{fig:theory_empirical_validation}, experiments strongly support this prediction.

    \item Based on this analysis, we introduce \texttt{RAMP} (\textbf{R}ecurrent \textbf{A}daptive \textbf{M}ixed \textbf{P}recision), a training-free mixed-precision inference method that performs early recurrent refinement with low-precision weight-and-activation computation and restores the original precision for terminal recurrent refinement.

    \item We show that \texttt{RAMP} preserves near-full-precision prefill quality across context lengths up to 4K while achieving up to \textbf{1.16$\times$ end-to-end first-prefill speedup} on an NVIDIA H100, corresponding to a \textbf{13.9\% reduction in latency}.
\end{itemize}

\begin{figure}[t]
    \centering
    \includegraphics[width=\linewidth]{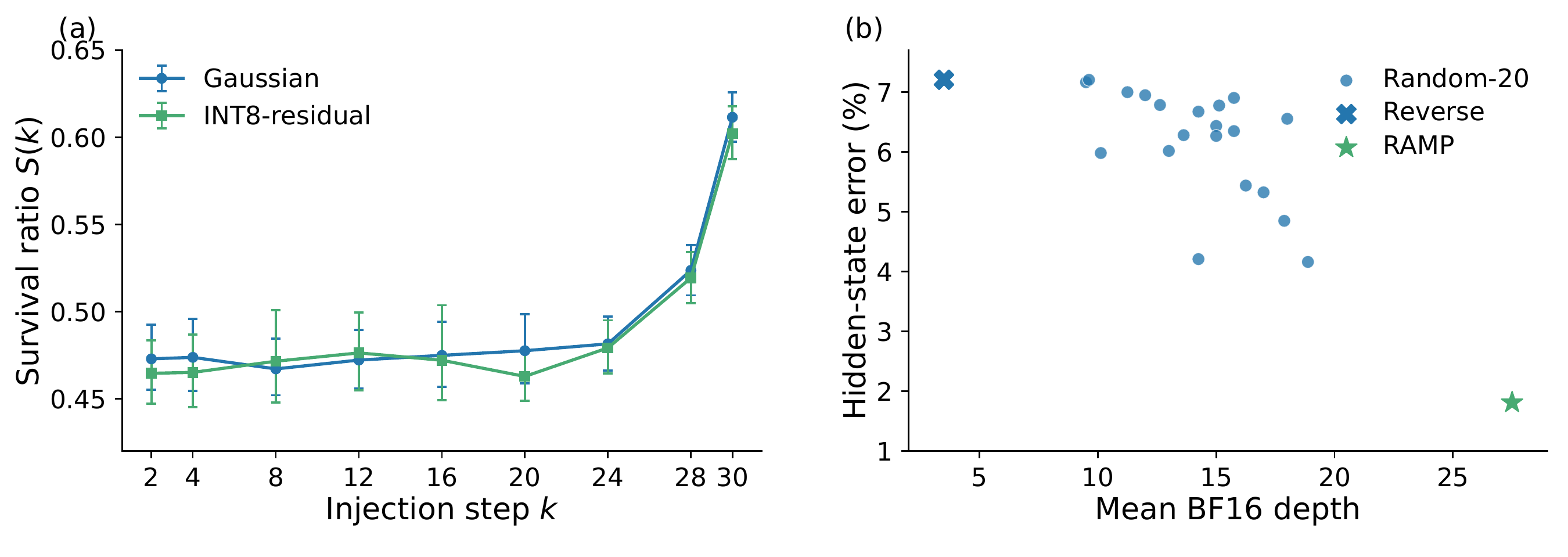}
    \caption{
    \textbf{Empirical validation of depth-dependent quantization error propagation.}
    \textbf{(a)} Perturbation survival as a function of injection depth. Both Gaussian perturbations and INT8 quantization-residual perturbations exhibit substantially lower survival when injected earlier in recurrent depth, while perturbations introduced near terminal depth retain a larger fraction of their initial error. Error bars denote the original 95\% confidence intervals.
    \textbf{(b)} Precision-placement comparison under the same budget of 24 FP8 and 8 BF16 recurrent steps. Each point represents one of 20 frozen random precision schedules, with the horizontal axis indicating the mean depth of its BF16 steps and the vertical axis showing final hidden-state relative error. Moving BF16 computation toward later recurrent depth is associated with lower final error. \texttt{RAMP} places all eight BF16 steps at the end and outperforms all 20 random schedules, whereas the reverse (\textsc{Early-BF16}) schedule places all eight BF16 steps at the beginning and produces substantially larger error.
    }
    \label{fig:theory_empirical_validation}
\end{figure}
\section{Problem Statement}

\subsection{Recurrent-Depth Language Models}
A recurrent-depth language model repeatedly applies a shared transformation to refine an internal representation across multiple recurrent steps. Let $\boldsymbol{h}^{(t)}$ denote the hidden representation after recurrent step $t$, with $\boldsymbol{h}^{(0)}$ produced by the model's input processing stage. Recurrent computation can be written as
\begin{equation}
\boldsymbol{h}^{(t+1)}
=
F_{\theta}\!\left(\boldsymbol{h}^{(t)}\right),
\qquad
t=0,\ldots,T-1,
\end{equation}
where $F_{\theta}$ is a shared recurrent block parameterized by $\theta$, and $T$ denotes the recurrent depth. Unlike a conventional depth-$T$ Transformer with distinct parameters at different layers, recurrent-depth models reuse the same parameters $\theta$ across recurrent steps. After the final recurrent step, the resulting representation $\boldsymbol{h}^{(T)}$ is processed by the model's output layers to produce token logits. In this work, we consider inference at a fixed recurrent depth $T$ and study the numerical precision used to execute the recurrent computation.

\subsection{Prefill in Recurrent-Depth Inference}
Given an input sequence of length $S$, the prefill stage processes the complete input sequence before autoregressive token generation begins. For a batch of size $B$, recurrent prefill operates on hidden representations containing $BS$ token positions. In a recurrent-depth model, the full-sequence representation is processed at every recurrent step, such that
\begin{equation}
\boldsymbol{h}^{(0)}
\rightarrow
\boldsymbol{h}^{(1)}
\rightarrow
\cdots
\rightarrow
\boldsymbol{h}^{(T)}.
\end{equation}
We define first-prefill latency as the time required to process the input sequence and produce the logits for the first generated token. Our primary efficiency metric is end-to-end first-prefill latency, and throughput is measured as the number of prompt tokens processed per second.

\subsection{Weight-and-Activation Quantization}
Weight-and-activation quantization reduces the numerical precision of both model weights and intermediate activations during inference. Consider a linear transformation
\begin{equation}
\boldsymbol{Y}
=
\boldsymbol{X}\boldsymbol{W},
\end{equation}
where $\boldsymbol{X}$ denotes the activation matrix, $\boldsymbol{W}$ denotes the weight matrix, and $\boldsymbol{Y}$ denotes the output matrix. Under weight-and-activation quantization, both operands are mapped to lower-precision numerical representations,
\begin{equation}
\widetilde{\boldsymbol{X}}
=
Q_A\!\left(\boldsymbol{X}\right),
\qquad
\widetilde{\boldsymbol{W}}
=
Q_W\!\left(\boldsymbol{W}\right),
\end{equation}
where $Q_A$ and $Q_W$ denote the activation and weight quantizers, respectively, and the corresponding low-precision computation is
\begin{equation}
\widetilde{\boldsymbol{Y}}
=
\widetilde{\boldsymbol{X}}
\widetilde{\boldsymbol{W}}.
\end{equation}
For recurrent-depth inference, let $p_t$ denote the numerical precision used at recurrent step $t$, and define a precision schedule as
\begin{equation}
\mathcal{P}
=
\left(p_0,p_1,\ldots,p_{T-1}\right).
\end{equation}
A uniform quantization schedule uses the same low-precision format at every recurrent step, whereas a mixed-precision schedule may use different numerical precisions across recurrent depth.
\section{Quantization Error Analysis}
\label{sec:error_analysis}

We analyze the error introduced by weight-and-activation quantization in recurrent-depth inference relative to a reference-precision trajectory. Our analysis addresses two questions: how numerical error introduced at one recurrent step propagates through subsequent recurrence, and how a fixed budget of quantized computation should be placed across recurrent depth. We require only local contraction of perturbations within the operating region relevant to quantized inference, rather than global contraction of the recurrent operator, and empirically examine this condition in our experiments. We first derive a general error bound for an arbitrary precision schedule and then show that, under a fixed quantization budget, placing low-precision computation earlier and reference-precision computation later minimizes this bound.

Let the reference-precision recurrent trajectory be
\begin{equation}
\boldsymbol{h}^{(t+1)}
=
F_{\theta}\!\left(\boldsymbol{h}^{(t)}\right),
\qquad
t=0,\ldots,T-1,
\label{eq:reference_recurrence}
\end{equation}
and let $\widetilde{F}_{\theta,t}$ denote the corresponding weight-and-activation quantized computation at recurrent step $t$. We define
\begin{equation}
e_t
=
\left\|
\widetilde{\boldsymbol{h}}^{(t)}
-
\boldsymbol{h}^{(t)}
\right\|
\label{eq:error_definition}
\end{equation}
as the deviation from the reference trajectory. Let $q_t\in\{0,1\}$ indicate whether recurrent step $t$ is quantized, where $q_t=1$ denotes quantized computation and $q_t=0$ denotes reference-precision computation, and let
\begin{equation}
\boldsymbol{q}
=
\left(q_0,q_1,\ldots,q_{T-1}\right)
\label{eq:schedule_indicator}
\end{equation}
denote the corresponding precision schedule.

\begin{theorem}[Quantization Error Propagation]
\label{thm:error_propagation}
Suppose that, within the relevant operating region, the reference recurrent operator satisfies
\begin{equation}
\left\|
F_{\theta}(\boldsymbol{x})
-
F_{\theta}(\boldsymbol{y})
\right\|
\leq
\rho_t
\left\|
\boldsymbol{x}
-
\boldsymbol{y}
\right\|,
\qquad
0\leq\rho_t<1,
\label{eq:local_contraction}
\end{equation}
and that each weight-and-activation quantized recurrent step satisfies the uniform perturbation bound
\begin{equation}
\left\|
\widetilde{F}_{\theta,t}(\boldsymbol{x})
-
F_{\theta}(\boldsymbol{x})
\right\|
\leq
\delta,
\qquad
\delta\geq0.
\label{eq:uniform_quantization_error}
\end{equation}
Then the recurrent deviation obeys
\begin{equation}
e_{t+1}
\leq
\rho_t e_t
+
q_t\delta.
\label{eq:one_step_schedule_error}
\end{equation}
If the quantized and reference trajectories share the same initial representation, such that $e_0=0$, the final deviation satisfies
\begin{equation}
e_T(\boldsymbol{q})
\leq
E(\boldsymbol{q})
\equiv
\delta
\sum_{i=0}^{T-1}
q_i
\prod_{j=i+1}^{T-1}
\rho_j.
\label{eq:general_schedule_bound}
\end{equation}
\end{theorem}

Theorem~\ref{thm:error_propagation} separates the final error into contributions from individual quantized recurrent steps. A perturbation introduced at step $i$ contributes to the final bound after being attenuated by all subsequent recurrent steps through the factor $\prod_{j=i+1}^{T-1}\rho_j$. Consequently, earlier quantization errors experience more subsequent contraction, whereas errors introduced near the end of recurrence retain a larger fraction of their effect. The proof is provided in Appendix~\ref{app:proof_error_propagation}.

\begin{theorem}[Optimal Precision Placement under a Fixed Quantization Budget]
\label{thm:optimal_placement}
Under the conditions of Theorem~\ref{thm:error_propagation}, consider all precision schedules containing exactly $K$ quantized recurrent steps,
\begin{equation}
\sum_{t=0}^{T-1}q_t
=
K.
\label{eq:fixed_quantization_budget}
\end{equation}
Define the low-to-high schedule used by \texttt{RAMP} as
\begin{equation}
q_t^{\mathrm{RAMP}}
=
\begin{cases}
1, & 0\leq t<K,\\
0, & K\leq t<T,
\end{cases}
\label{eq:ramp_schedule_theory}
\end{equation}
and the reverse schedule as
\begin{equation}
q_t^{\mathrm{Rev}}
=
\begin{cases}
0, & 0\leq t<T-K,\\
1, & T-K\leq t<T.
\end{cases}
\label{eq:reverse_schedule_theory}
\end{equation}
Then, for every schedule $\boldsymbol{q}$ satisfying Equation~\ref{eq:fixed_quantization_budget},
\begin{equation}
E\!\left(\boldsymbol{q}^{\mathrm{RAMP}}\right)
\leq
E(\boldsymbol{q})
\leq
E\!\left(\boldsymbol{q}^{\mathrm{Rev}}\right).
\label{eq:optimal_schedule_ordering}
\end{equation}
Thus, among all schedules with the same number of quantized recurrent steps, placing quantized computation as early as possible minimizes the theoretical final-error bound, while placing it as late as possible maximizes the bound.
\end{theorem}

Theorem~\ref{thm:optimal_placement} provides the central theoretical motivation for \texttt{RAMP}: its advantage does not arise merely from using fewer quantized steps than full quantization, but from where low- and high-precision computation are placed across recurrent depth. Early quantization errors receive the largest amount of subsequent attenuation, while preserving reference precision for terminal refinement avoids introducing the most weakly attenuated errors. Random or interleaved schedules with the same quantization budget lie between the low-to-high and high-to-low extremes in the resulting bound. As a direct consequence, \texttt{RAMP} also has a tighter bound than full quantization,
\begin{equation}
E_{\mathrm{AllQ}}
-
E_{\mathrm{RAMP}}
=
\delta
\sum_{i=K}^{T-1}
\prod_{j=i+1}^{T-1}
\rho_j
\geq
0,
\label{eq:allq_ramp_consequence}
\end{equation}
where
\begin{equation}
E_{\mathrm{AllQ}}
=
\delta
\sum_{i=0}^{T-1}
\prod_{j=i+1}^{T-1}
\rho_j
\label{eq:allq_bound}
\end{equation}
and
\begin{equation}
E_{\mathrm{RAMP}}
=
\delta
\sum_{i=0}^{K-1}
\prod_{j=i+1}^{T-1}
\rho_j.
\label{eq:ramp_bound}
\end{equation}
The proof of Theorem~\ref{thm:optimal_placement} is provided in Appendix~\ref{app:proof_optimal_placement}.
\begin{figure}[t]
    \centering
    \includegraphics[width=\linewidth]{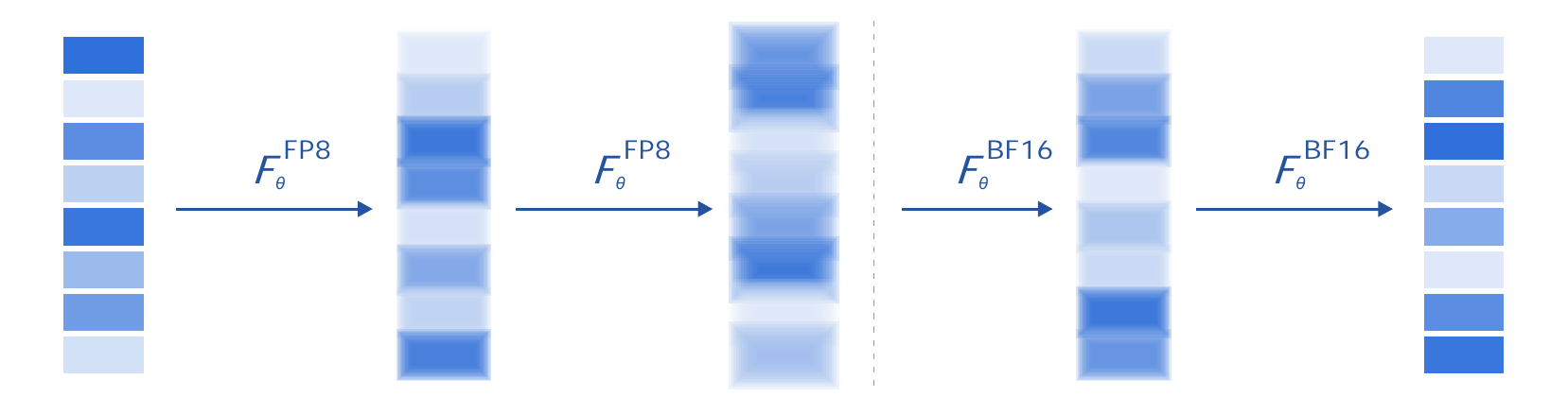}
    \caption{
    \textbf{Overview of \texttt{RAMP}.}
    \texttt{RAMP} executes the shared recurrent block $F_{\theta}$ using FP8 weight-and-activation computation during early recurrent refinement and restores BF16 computation for the final recurrent steps. The dashed line marks the precision switch, while the recurrent architecture and total recurrent depth remain unchanged.
    }
    \label{fig:ramp_method}
\end{figure}

\section{RAMP}
\label{sec:ramp}

Theorem~\ref{thm:optimal_placement} shows that, under a fixed precision budget, placing low-precision computation earlier in recurrent depth minimizes the final-error upper bound. Motivated by this result, we introduce \texttt{RAMP} (\textbf{R}ecurrent \textbf{A}daptive \textbf{M}ixed \textbf{P}recision), a training-free inference method that performs early recurrent refinement with low-precision computation and preserves the original precision for terminal refinement.

\subsection{Depth-Adaptive Precision Schedule}
Let $T$ denote the total recurrent depth and $K<T$ denote the precision-switching step. \texttt{RAMP} defines the recurrent precision schedule as
\begin{equation}
p_t
=
\begin{cases}
p_{\mathrm{low}}, & 0 \leq t < K,\\
p_{\mathrm{ref}}, & K \leq t < T,
\end{cases}
\end{equation}
where $p_{\mathrm{low}}$ denotes the low-precision format and $p_{\mathrm{ref}}$ denotes the model's original inference precision. The corresponding recurrent computation is
\begin{equation}
\widetilde{\boldsymbol{h}}^{(t+1)}
=
\begin{cases}
\widetilde{F}_{\theta}
\!\left(
\widetilde{\boldsymbol{h}}^{(t)}
\right),
& 0 \leq t < K,\\
F_{\theta}
\!\left(
\widetilde{\boldsymbol{h}}^{(t)}
\right),
& K \leq t < T,
\end{cases}
\end{equation}
where $\widetilde{F}_{\theta}$ denotes the weight-and-activation quantized recurrent operator and $F_{\theta}$ denotes the reference-precision recurrent operator. \texttt{RAMP} therefore differs from uniform quantization only in the terminal recurrent steps: instead of continuing to inject quantization error throughout the complete trajectory, it restores reference-precision computation after step $K$.

\subsection{FP8--BF16 Instantiation}
\label{FP8-BF16}
In our implementation, we instantiate $p_{\mathrm{low}}$ with FP8 and $p_{\mathrm{ref}}$ with BF16. During the low-precision phase, the eligible dense linear operations in the recurrent block use FP8 weights and dynamically quantized FP8 activations with native low-precision matrix multiplication, while the recurrent hidden representation and the surrounding operations remain in BF16. At step $K$, the recurrent block switches back to its original BF16 computation and remains in BF16 for all subsequent recurrent steps. Unless otherwise specified, we use $T=32$ and $K=24$, corresponding to 24 FP8 recurrent steps followed by 8 BF16 recurrent steps. \texttt{RAMP} requires no retraining, does not modify the recurrent architecture or recurrent depth, and introduces only a single precision transition during inference.
\section{Related Work}

\subsection{Recurrent-Depth and Looped Language Models}
Recurrent computation across network depth has a long history, with Universal Transformers introducing parameter sharing and adaptive recurrence across Transformer depth \citep{dehghani2018universal}. More recent work has revisited looped or recurrent-depth architectures as a mechanism for scaling test-time computation without explicitly generating intermediate reasoning tokens. ~\citet{saunshi2025reasoning} show that looped Transformers can emulate substantially deeper networks on reasoning tasks, while ~\citet{geiping2026scaling} introduce a recurrent-depth language model whose performance improves with additional recurrent computation. Subsequent work studies elastic recurrent depth, recurrent stability, and latent reasoning behavior \citep{jeddi2026loopformer,yang2026stabilizing}. In contrast to these works, which primarily study modeling capability or recurrent dynamics, we focus on the inference cost of recurrent-depth prefill and exploit depth-dependent numerical sensitivity to reduce this cost.

\subsection{Quantized Large Language Model Inference}
Quantization reduces inference cost by representing model computation at lower numerical precision. Early methods focus on efficient INT8 execution and weight compression \citep{lin2024awq}, while weight-and-activation quantization directly lowers the precision of dense arithmetic. SmoothQuant enables training-free W8A8 inference by mitigating activation outliers \citep{xiao2023smoothquant}, and subsequent systems such as Atom and QServe extend low-bit weight-and-activation execution toward higher-throughput LLM serving \citep{zhao2024atom,lin2025qserve}. Recent methods such as QuaRot and SpinQuant further improve aggressive weight-and-activation quantization by transforming representations to reduce quantization difficulty \citep{ashkboos2024quarot,ICLR2025_e5b1c0d4}. These approaches primarily seek a quantization configuration that can be applied throughout a feed-forward model. \texttt{RAMP} instead exploits recurrent depth as an additional precision-allocation dimension, assigning different precision to different executions of the same shared block.

\subsection{Mixed-Precision and Prefill-Efficient Inference}
Mixed-precision quantization allocates different numerical precisions to different parts of a network in order to balance accuracy and hardware efficiency. Prior methods typically assign precision across layers, operators, channels, or tensors according to sensitivity or hardware cost \citep{yao2021hawq,zhao2024atom}. \texttt{RAMP} differs in that the same recurrent block is executed repeatedly, and precision is allocated across recurrent time rather than across distinct network components. Our setting is also motivated by the distinct system characteristics of LLM prefill: prefill processes many prompt tokens in parallel and is typically compute-intensive, while its latency determines time to first token \citep{agrawal2024taming,zhong2024distserve,patel2024splitwise}. We therefore focus specifically on reducing recurrent-depth prefill latency through depth-adaptive numerical precision.
\section{Experiments}
\label{sec:experiments}

\subsection{Experimental Setup}

\paragraph{Model and precision configuration.}
We evaluate \texttt{RAMP} on Huginn-0125 \citep{geiping2026scaling}, using the default recurrent depth $T=32$. The model consists of a prelude, a shared recurrent block repeatedly applied across recurrent depth, and a coda that produces the final token logits. Unless otherwise specified, \texttt{RAMP} uses a single-precision transition at $K=24$: the first 24 recurrent steps execute the eligible dense linear operations using FP8 weight-and-activation computation, while the final 8 recurrent steps restore the original BF16 computation. The recurrent hidden state, prelude, coda, and operations outside the quantized recurrent linears remain in BF16. 
End-to-end efficiency is evaluated separately on an NVIDIA H100 using native FP8 weight-and-activation arithmetic.

\paragraph{Datasets and evaluation cohorts.}
Our primary language-modeling fidelity evaluation uses frozen WikiText sequences, with PG19 serving as an independent long-context generalization dataset. For WikiText, we construct 32 fixed evaluation endpoints and form nested input contexts of $S\in\{128,256,512,1024,2048,4096\}$ tokens ending at the same corpus location. For each endpoint, the final 64 target tokens are identical across context lengths, so increasing $S$ changes only the amount of preceding context rather than the evaluated continuation. For PG19, we similarly select 32 frozen validation books and construct document-contained nested contexts of $S\in\{512,1024,2048,4096\}$ with identical tail-64 prediction targets across context lengths. This paired construction is used for the long-context fidelity results in Section~6.2 and prevents context-length comparisons from being confounded by changes in the prediction target. The fixed-budget precision-placement experiment in Figure~\ref{fig:theory_empirical_validation} uses the earlier frozen $S=128$ WikiText cohort for language-modeling loss together with a fixed 12-prompt fidelity cohort for hidden-state and logit comparisons. The perturbation-recoverability experiment in Figure~\ref{fig:theory_empirical_validation} uses a separate frozen 24-prompt diagnostic cohort and is used only to test the error-propagation mechanism rather than downstream language-modeling quality.

\paragraph{Baselines.}
We compare \texttt{RAMP} against four primary baselines that isolate different explanations for its behavior. \textbf{BF16-32} executes all 32 recurrent steps in BF16 and serves as the full-precision reference. \textbf{BF16-8} executes only 8 recurrent steps, all in BF16, and serves as a reduced-depth control for testing whether the first 24 low-precision steps in \texttt{RAMP} provide useful recurrent refinement rather than merely adding dispensable computation. \textbf{All-FP8} executes all 32 recurrent steps with FP8 weight-and-activation computation and measures the effect of uniform quantization throughout recurrence. \textbf{Early-BF16} executes 8 BF16 steps followed by 24 FP8 steps; it uses exactly the same 24-FP8/8-BF16 precision budget as \texttt{RAMP} but reverses their placement, providing a matched-budget control for precision ordering. Our method, \textbf{\texttt{RAMP}}, executes 24 FP8 steps followed by 8 BF16 steps. Except for the intentionally shallower BF16-8 control, all primary configurations use the same model parameters, recurrent depth, inputs, and number of recurrent applications.

\paragraph{Quality metrics.}
We measure fidelity relative to the BF16-32 trajectory. Our primary metric is the paired change in teacher-forced negative log-likelihood over the final 64 target tokens, denoted $\Delta$NLL, where each configuration is compared against BF16 on exactly the same endpoint and target tokens. We additionally report final hidden-state relative $\ell_2$ error and final-logit relative $\ell_2$ error to distinguish representation-level deviation from output-level deviation. Where available under the same evaluation protocol, we also retain KL divergence between the predictive distributions and top-1 token agreement. Confidence intervals for $\Delta$NLL are obtained by bootstrapping over the frozen evaluation endpoints. We emphasize these continuous fidelity measures because numerical deviations introduced by quantization may be detectable before they change the model's top-1 prediction.

\paragraph{Precision-placement evaluation.}
To test the fixed-budget prediction of Theorem~\ref{thm:optimal_placement} independently of the hand-designed \texttt{RAMP} and Early-BF16 schedules, we additionally evaluate 20 frozen random precision schedules. Every random schedule contains exactly 24 FP8 and 8 BF16 recurrent steps; the eight BF16 positions are sampled uniformly before any outcomes are observed, and neither \texttt{RAMP} nor Early-BF16 is included in the random set. We compare the schedules using paired WikiText NLL, final hidden-state error, and final logit error, and summarize precision placement by the mean recurrent depth of the BF16 steps. The perturbation experiment in Figure~\ref{fig:theory_empirical_validation} provides a complementary mechanism-level test: perturbations are injected at multiple recurrent depths, and their residual error after the final recurrent step is measured to quantify depth-dependent attenuation.

\paragraph{Systems evaluation.}
We evaluate end-to-end first-prefill efficiency on an NVIDIA H100 using native FP8 weight-and-activation arithmetic. We measure single-request prefill over context lengths $S\in\{128,256,512,1024,2048,4096\}$ and compare \texttt{RAMP} directly against the BF16-32 reference. Timing begins with token IDs already resident on the GPU and an empty attention cache, and ends when the model produces the logits required for the first generated token. Tokenization, dataset loading, host-to-device transfer, network overhead, and subsequent autoregressive decoding are excluded. We report first-prefill latency, prompt-token throughput, and speedup relative to BF16. Both BF16 and \texttt{RAMP} use the same semantics-preserving optimized empty-cache prefill path, which removes per-token cache bookkeeping while leaving the model computation and produced outputs unchanged. The reported speedups therefore reflect the realized effect of numerical precision under the same optimized prefill implementation rather than an implementation asymmetry between the two configurations.

\begin{figure}[t]
    \centering
    \begin{subfigure}[t]{0.48\linewidth}
        \centering
        \includegraphics[width=\linewidth]{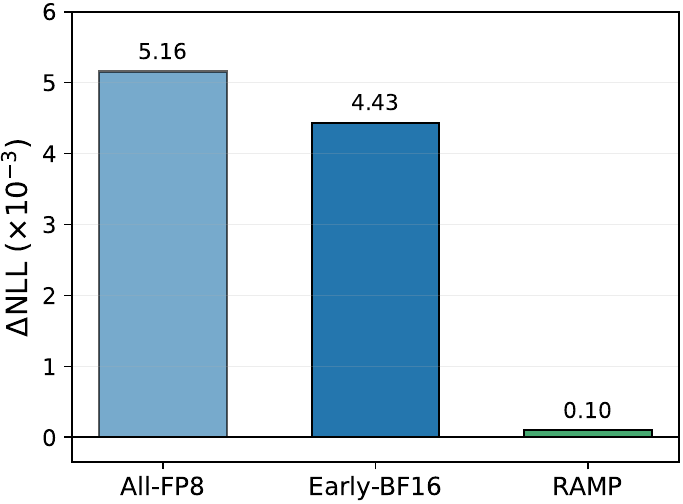}
        \caption{}
        \label{fig:4K}
    \end{subfigure}
    \hfill
    \begin{subfigure}[t]{0.48\linewidth}
        \centering
        \includegraphics[width=\linewidth]{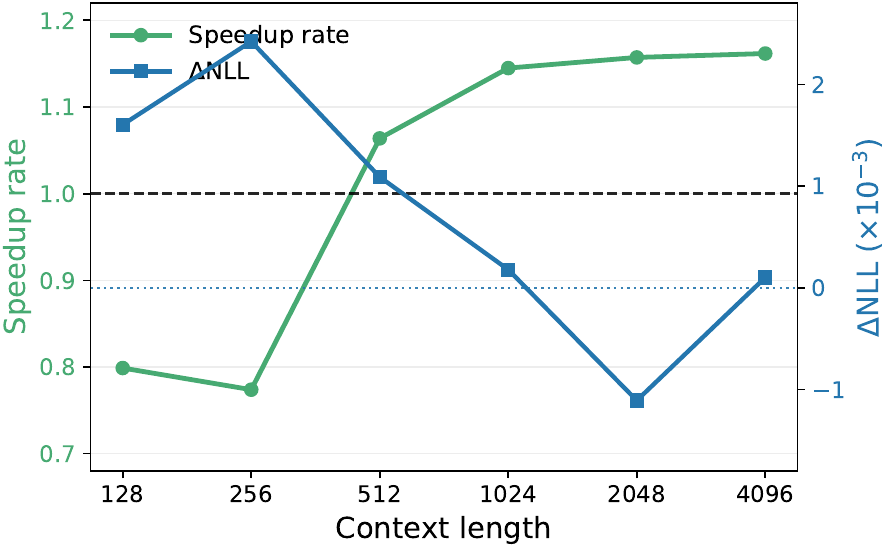}
        \caption{}
        \label{fig:ramp_context_scaling}
    \end{subfigure}
    \caption{
    \textbf{Quality preservation and end-to-end prefill efficiency of \texttt{RAMP}.}
    \textbf{(a)} WikiText prefill fidelity at 4K context under different precision schedules. \texttt{RAMP} uses 24 FP8 recurrent steps followed by 8 BF16 steps, while Early-BF16 uses the same precision budget in the opposite order. Restoring BF16 computation for terminal refinement substantially reduces the NLL deviation from BF16-32.
    \textbf{(b)} WikiText fidelity and H100 first-prefill efficiency across context lengths. The green curve reports end-to-end first-prefill speedup over BF16, while the blue curve reports paired WikiText tail-64 $\Delta$NLL relative to BF16-32. \texttt{RAMP} preserves near-BF16 fidelity across context lengths up to 4K while reaching approximately $1.16\times$ first-prefill speedup at long context.
    }
    \label{fig:ramp_main_results}
\end{figure}

\subsection{\texttt{RAMP} Preserves Full-Precision Prefill Quality Across Datasets}

\begin{table}[t]
    \centering
    \caption{
    \textbf{Prefill fidelity at 4K context on WikiText and PG19.}
    All metrics are measured relative to the full-depth BF16-32 reference
    within each dataset. BF16-8 is a reduced-depth control that executes only
    8 recurrent steps, all in BF16. All-FP8, Early-BF16, and \texttt{RAMP}
    execute the full 32 recurrent steps; Early-BF16 and \texttt{RAMP} use the
    same budget of 24 FP8 and 8 BF16 recurrent steps but place the BF16
    computation at opposite ends of recurrent depth.
    Lower is better for error metrics, while higher is better for top-1
    agreement.
    }
    \label{tab:quality_4k}
    \small
    \setlength{\tabcolsep}{4.2pt}
    \renewcommand{\arraystretch}{1.08}
    \begin{tabular}{@{}lrrrrrrrr@{}}
        \toprule
        &
        \multicolumn{4}{c}{\textbf{WikiText}}
        &
        \multicolumn{4}{c}{\textbf{PG19}} \\
        \cmidrule(lr){2-5}
        \cmidrule(lr){6-9}

        \textbf{Schedule}
        & $\Delta$\textbf{NLL} $\downarrow$
        & \textbf{Hidden} $\ell_2$ $\downarrow$
        & \textbf{Logit} $\ell_2$ $\downarrow$
        & \textbf{Top-1} $\uparrow$
        & $\Delta$\textbf{NLL} $\downarrow$
        & \textbf{Hidden} $\ell_2$ $\downarrow$
        & \textbf{Logit} $\ell_2$ $\downarrow$
        & \textbf{Top-1} $\uparrow$ \\
        \midrule

        BF16-32
        & 0
        & 0
        & 0
        & 100.00\%
        & 0
        & 0
        & 0
        & 100.00\% \\

        BF16-8
        & +0.12461
        & 36.14\%
        & 15.76\%
        & 83.25\%
        & +0.11260
        & 37.68\%
        & 15.13\%
        & 79.79\% \\

        All-FP8
        & +0.00516
        & 8.18\%
        & 3.80\%
        & 96.63\%
        & +0.00233
        & 8.45\%
        & 3.57\%
        & 95.61\% \\

        Early-BF16
        & +0.00444
        & 8.15\%
        & 3.78\%
        & 95.95\%
        & +0.00325
        & 8.44\%
        & 3.58\%
        & 95.61\% \\

        \textbf{\texttt{RAMP}}
        & \textbf{+0.00010}
        & \textbf{2.65\%}
        & \textbf{1.34\%}
        & \textbf{98.73\%}
        & \textbf{+0.00007}
        & \textbf{3.08\%}
        & \textbf{1.35\%}
        & \textbf{98.29\%} \\

        \bottomrule
    \end{tabular}
\end{table}

\paragraph{RAMP consistently recovers the full-precision trajectory across datasets.}
Table~\ref{tab:quality_4k} compares the main precision schedules at 4K context on WikiText and PG19. The same qualitative pattern appears on the two independently constructed evaluation cohorts. Uniform FP8 computation produces substantially larger deviations in the final hidden representation and output logits, whereas \texttt{RAMP} retains low-precision execution for most of the recurrent trajectory while remaining much closer to the BF16 reference. The agreement across language-modeling loss, representation error, logit error, and top-1 agreement shows that this recovery is neither metric-specific nor restricted to a single corpus.

\paragraph{Precision placement, rather than precision budget alone, determines fidelity.}
Early-BF16 provides a matched-budget control for the placement of high-precision computation: it uses exactly the same number of FP8 and BF16 recurrent steps as \texttt{RAMP}, but reverses their order. On both WikiText and PG19, placing BF16 computation before the quantized phase leaves the final trajectory close to uniform FP8 execution, whereas reserving the same high-precision budget for terminal refinement substantially reduces representation- and logit-level error. On PG19, the likelihood difference between the two matched-budget schedules is not statistically resolved, but the representation-level ordering remains pronounced. Together, the two datasets support the depth asymmetry predicted by Theorem~\ref{thm:optimal_placement}: early quantization errors have more subsequent recurrent computation over which they can be attenuated, while preserving high precision near terminal depth avoids introducing weakly recoverable errors late in the trajectory.

\paragraph{Early low-precision recurrent steps perform useful refinement.}
BF16-8 provides a reduced-depth control for testing whether the first 24 low-precision recurrent steps in \texttt{RAMP} can simply be removed.
Across both WikiText and PG19, executing only eight BF16 recurrent steps produces substantially larger likelihood, representation, and logit deviations than the full-depth \texttt{RAMP} trajectory. 
This behavior rules out the interpretation that only the final eight high-precision steps are responsible for \texttt{RAMP}'s fidelity.
Instead, the preceding 24 FP8 steps continue to perform useful recurrent refinement, while the terminal BF16 phase
recovers the trajectory toward the full-precision solution.


\paragraph{Quality remains stable across long contexts.}
We additionally evaluate \texttt{RAMP} using nested contexts while holding the prediction target fixed. On WikiText, paired $\Delta$NLL remains close to zero from 128 to 4096 tokens with no systematic growth, as shown in Figure~\ref{fig:ramp_context_scaling}. On PG19, \texttt{RAMP} likewise remains close to BF16-32 over 512--4096 tokens: hidden-state and logit errors remain stable, while $\Delta$NLL stays small with a mild positive trend toward longer contexts. Thus, across both datasets, increasing the amount of full-sequence recurrent computation performed under FP8 does not lead to a corresponding growth in representation-level deviation over the evaluated long-context range.

\subsection{End-to-End Long-Context Prefill Efficiency}

\begin{table}[t]
    \centering
    \caption{
    \textbf{End-to-end first-prefill latency across context lengths on an NVIDIA H100.}
    BF16 and \texttt{RAMP} use the same optimized empty-cache prefill path.
    Speedup is computed as BF16 latency divided by \texttt{RAMP} latency; values above $1\times$ indicate faster execution with \texttt{RAMP}.
    }
    \label{tab:prefill_efficiency}
    \small
    \setlength{\tabcolsep}{9pt}
    \renewcommand{\arraystretch}{1.08}
    \begin{tabular}{@{}rrrr@{}}
        \toprule
        \textbf{Context Length}
        & \textbf{BF16 (ms)}
        & \textbf{\texttt{RAMP} (ms)}
        & \textbf{Speedup} $\uparrow$ \\
        \midrule
        128
        & 131.784
        & 164.972
        & 0.7988$\times$ \\

        256
        & 152.802
        & 197.474
        & 0.7738$\times$ \\

        512
        & 148.473
        & 139.567
        & 1.0638$\times$ \\

        1024
        & 274.090
        & 239.395
        & 1.1449$\times$ \\

        2048
        & 543.158
        & 469.357
        & 1.1572$\times$ \\

        4096
        & 1103.444
        & 949.847
        & 1.1617$\times$ \\
        \bottomrule
    \end{tabular}
\end{table}

\paragraph{RAMP exhibits a clear long-context efficiency regime.}
Table~\ref{tab:prefill_efficiency} reports end-to-end first-prefill latency as context length increases from 128 to 4096 tokens. At short contexts, \texttt{RAMP} is slower than BF16 because low-precision execution overhead is large relative to the recurrent matrix computation being accelerated. As the context becomes longer, this overhead is increasingly amortized by the larger full-sequence recurrent workload. The crossover occurs between 256 and 512 tokens, after which \texttt{RAMP} becomes faster than BF16, with the gain becoming stable from 1K tokens onward.
Fidelity results are shown in Appendix~\ref{app:context_scaling}.

\paragraph{RAMP reaches a stable $\mathbf{1.16\times}$ end-to-end speedup at long context.}
At 4K context, first-prefill latency decreases from 1103.444\,ms with BF16 to 949.847\,ms with \texttt{RAMP}, corresponding to a $1.1617\times$ speedup and a 13.9\% reduction in latency. Similar gains are already observed at 1K and 2K contexts, where \texttt{RAMP} achieves $1.1449\times$ and $1.1572\times$ speedup, respectively. The consistent long-context gains show that the 4K result is not an isolated operating point: once the recurrent prefill workload is sufficiently large, native FP8 computation converts the reduced arithmetic precision of the early recurrent steps into stable end-to-end savings.

\paragraph{The systems gain grows while WikiText fidelity remains stable.}
Figure~\ref{fig:ramp_context_scaling} juxtaposes H100 first-prefill efficiency with the paired WikiText fidelity measurements from Section~6.2. As context length increases, the first-prefill speedup rises and approaches $1.16\times$, while the corresponding WikiText $\Delta$NLL remains close to zero. Combined with the independent PG19 fidelity results in Section~6.2, these measurements show that the long-context regime in which low-precision execution becomes most beneficial does not require a corresponding loss of prefill fidelity.

\subsection{Precision Placement Matters}

\begin{table}[t]
    \centering
    \caption{
    \textbf{Fixed-budget precision-placement comparison on the fixed 12-prompt fidelity cohort.}
    All configurations use exactly 24 FP8 and 8 BF16 recurrent steps.
    Random-20 reports the mean hidden-state error across the frozen random schedules, while Early-BF16 and \texttt{RAMP} place all BF16 computation at the beginning and end of recurrent depth, respectively.
    }
    \label{tab:precision_placement}
    \small
    \setlength{\tabcolsep}{9pt}
    \renewcommand{\arraystretch}{1.08}
    \begin{tabular}{@{}lcc@{}}
        \toprule
        \textbf{Schedule}
        & \textbf{BF16 Placement}
        & \textbf{Hidden} $\ell_2$ $\downarrow$ \\
        \midrule
        Random-20 (mean)
        & Random
        & 6.17\% \\

        Early-BF16
        & Steps 0--7
        & 7.20\% \\

        \textbf{\texttt{RAMP}}
        & \textbf{Steps 24--31}
        & \textbf{1.81\%} \\
        \bottomrule
    \end{tabular}
\end{table}

\paragraph{RAMP consistently outperforms random precision placement.}
We isolate the effect of precision placement by holding the precision budget fixed and varying only which recurrent steps use BF16. As summarized in Table~\ref{tab:precision_placement}, \texttt{RAMP} produces substantially lower final hidden-state error on the fixed 12-prompt fidelity cohort than either the random schedules or Early-BF16. The same qualitative advantage holds for paired NLL on the frozen $S=128$ WikiText cohort and for final-logit error on the prompt cohort, with \texttt{RAMP} outperforming every frozen random schedule on all three fidelity measures. Thus, the benefit of \texttt{RAMP} cannot be explained merely by the presence of a fixed number of high-precision recurrent steps.

\paragraph{Later high-precision computation systematically improves fidelity.}
Figure~\ref{fig:theory_empirical_validation} reveals a clear relationship between BF16 placement and final hidden-state error across the frozen random schedules. Schedules that place their BF16 steps later in recurrent depth tend to achieve higher fidelity, while schedules that allocate high precision earlier retain larger terminal deviations. \texttt{RAMP} and Early-BF16 occupy the two placement extremes, providing a matched-budget comparison in which only the direction of precision allocation changes.

\paragraph{Observed precision ordering agrees with the theoretical prediction.}
These results provide direct empirical support for Theorem~\ref{thm:optimal_placement}. Under the local contraction model, perturbations introduced earlier have more subsequent recurrent steps over which they can be attenuated, whereas perturbations introduced near terminal depth have less opportunity to recover. The observed schedule ordering follows the same prediction: low precision is most effectively allocated to early recurrent refinement, while high precision is most valuable near the end of the trajectory.

\subsection{Precision-Switch Ablation}

\begin{table}[t]
    \centering
    \caption{
    \textbf{Effect of the terminal BF16 cleanup length on WikiText at 4K context.}
    All configurations use the same 32 recurrent steps while varying the precision-switch point.
    Lower is better for both metrics.
    }
    \label{tab:switch_ablation}
    \small
    \setlength{\tabcolsep}{10pt}
    \renewcommand{\arraystretch}{1.08}
    \begin{tabular}{@{}lcc@{}}
        \toprule
        \textbf{FP8 $\rightarrow$ BF16 Steps}
        & $\Delta$\textbf{NLL} $\downarrow$
        & \textbf{Hidden} $\ell_2$ $\downarrow$ \\
        \midrule
        32 $\rightarrow$ 0
        & +0.00516
        & 8.18\% \\

        \textbf{24 $\rightarrow$ 8}
        & \textbf{+0.00010}
        & \textbf{2.65\%} \\

        16 $\rightarrow$ 16
        & $-0.00016$
        & 1.91\% \\

        0 $\rightarrow$ 32
        & 0
        & 0 \\
        \bottomrule
    \end{tabular}
\end{table}

\paragraph{A short terminal BF16 phase recovers most of the lost fidelity.}
Table~\ref{tab:switch_ablation} varies the precision-switch point on the 4K WikiText evaluation while keeping the total recurrent depth fixed. Uniform FP8 execution produces a clear deviation from the BF16 reference, whereas restoring BF16 for only the final eight recurrent steps removes most of this error. Extending the BF16 phase further continues to improve representation fidelity, but the additional recovery beyond the default \texttt{RAMP} schedule is comparatively small.

\paragraph{The default switch balances fidelity and low-precision coverage.}
The ablation highlights the tradeoff underlying the choice of $K=24$. Moving the switch earlier allocates more recurrent computation to BF16 and therefore provides additional numerical recovery, but correspondingly reduces the fraction of recurrent work executed in FP8. The default \texttt{RAMP} configuration preserves low-precision execution for three quarters of the recurrent trajectory while already recovering most of the fidelity gap to full BF16, making it the operating point used throughout the main experiments.
\section{Conclusion}

We study mixed-precision inference for recurrent-depth language models from the perspective of where quantized computation should be placed across recurrent depth. Our analysis shows that, under locally contractive recurrent dynamics, perturbations introduced earlier receive more subsequent attenuation, yielding a smaller final-error upper bound than perturbations introduced near terminal depth. This depth asymmetry is strongly supported by perturbation experiments and fixed-budget precision-placement comparisons. Motivated by this observation, we introduce \texttt{RAMP}, a training-free inference method that performs early recurrent refinement using FP8 weight-and-activation computation and restores BF16 precision for terminal refinement. Across context lengths up to 4K, \texttt{RAMP} preserves near-full-precision prefill quality while converting low-precision recurrent computation into up to $1.16\times$ end-to-end first-prefill speedup on an NVIDIA H100, corresponding to a 13.9\% latency reduction. These results suggest that recurrent depth is not only an additional axis for scaling computation, but also an opportunity for structured precision allocation: early recurrent computation can be executed aggressively in low precision, while a short high-precision terminal phase can recover most of the fidelity of full-precision inference.
\newpage

\bibliography{references}

@article{geiping2026scaling,
  title={Scaling up test-time compute with latent reasoning: A recurrent depth approach},
  author={Geiping, Jonas and McLeish, Sean and Jain, Neel and Kirchenbauer, John and Singh, Siddharth and Bartoldson, Brian and Kailkhura, Bhavya and Bhatele, Abhinav and Goldstein, Tom},
  journal={Advances in Neural Information Processing Systems},
  volume={38},
  pages={41340--41391},
  year={2026}
}

@article{lin2025qserve,
  title={Qserve: W4a8kv4 quantization and system co-design for efficient llm serving},
  author={Lin, Yujun and Tang, Haotian and Yang, Shang and Zhang, Zhekai and Xiao, Guangxuan and Gan, Chuang and Han, Song},
  journal={Proceedings of Machine Learning and Systems},
  volume={7},
  year={2025}
}

@inproceedings{ICLR2025_e5b1c0d4,
 author = {Liu, Zechun and Zhao, Changsheng and Fedorov, Igor and Soran, Bilge and Choudhary, Dhruv and Krishnamoorthi, Raghuraman and Chandra, Vikas and Tian, Yuandong and Blankevoort, Tijmen},
 booktitle = {International Conference on Learning Representations},
 editor = {Y. Yue and A. Garg and N. Peng and F. Sha and R. Yu},
 pages = {92009--92032},
 title = {SpinQuant: LLM Quantization with Learned Rotations},
 url = {https://proceedings.iclr.cc/paper_files/paper/2025/file/e5b1c0d4866f72393c522c8a00eed4eb-Paper-Conference.pdf},
 volume = {2025},
 year = {2025}
}

@article{dettmers2022gpt3,
  title={Gpt3. int8 (): 8-bit matrix multiplication for transformers at scale},
  author={Dettmers, Tim and Lewis, Mike and Belkada, Younes and Zettlemoyer, Luke},
  journal={Advances in neural information processing systems},
  volume={35},
  pages={30318--30332},
  year={2022}
}

@inproceedings{zhong2024distserve,
  title={$\{$DistServe$\}$: Disaggregating prefill and decoding for goodput-optimized large language model serving},
  author={Zhong, Yinmin and Liu, Shengyu and Chen, Junda and Hu, Jianbo and Zhu, Yibo and Liu, Xuanzhe and Jin, Xin and Zhang, Hao},
  booktitle={18th USENIX symposium on operating systems design and implementation (OSDI 24)},
  pages={193--210},
  year={2024}
}

@inproceedings{patel2024splitwise,
  title={Splitwise: Efficient generative llm inference using phase splitting},
  author={Patel, Pratyush and Choukse, Esha and Zhang, Chaojie and Shah, Aashaka and Goiri, {\'I}{\~n}igo and Maleki, Saeed and Bianchini, Ricardo},
  booktitle={2024 ACM/IEEE 51st Annual International Symposium on Computer Architecture (ISCA)},
  pages={118--132},
  year={2024},
  organization={IEEE}
}

@inproceedings{agrawal2024taming,
  title={Taming $\{$Throughput-Latency$\}$ tradeoff in $\{$LLM$\}$ inference with $\{$Sarathi-Serve$\}$},
  author={Agrawal, Amey and Kedia, Nitin and Panwar, Ashish and Mohan, Jayashree and Kwatra, Nipun and Gulavani, Bhargav and Tumanov, Alexey and Ramjee, Ramachandran},
  booktitle={18th USENIX symposium on operating systems design and implementation (OSDI 24)},
  pages={117--134},
  year={2024}
}

@inproceedings{yao2021hawq,
  title={Hawq-v3: Dyadic neural network quantization},
  author={Yao, Zhewei and Dong, Zhen and Zheng, Zhangcheng and Gholami, Amir and Yu, Jiali and Tan, Eric and Wang, Leyuan and Huang, Qijing and Wang, Yida and Mahoney, Michael and others},
  booktitle={International conference on machine learning},
  pages={11875--11886},
  year={2021},
  organization={PMLR}
}

@article{ashkboos2024quarot,
  title={Quarot: Outlier-free 4-bit inference in rotated llms},
  author={Ashkboos, Saleh and Mohtashami, Amirkeivan and Croci, Maximilian L and Li, Bo and Cameron, Pashmina and Jaggi, Martin and Alistarh, Dan and Hoefler, Torsten and Hensman, James},
  journal={Advances in Neural Information Processing Systems},
  volume={37},
  pages={100213--100240},
  year={2024}
}

@article{zhao2024atom,
  title={Atom: Low-bit quantization for efficient and accurate llm serving},
  author={Zhao, Yilong and Lin, Chien-Yu and Zhu, Kan and Ye, Zihao and Chen, Lequn and Zheng, Size and Ceze, Luis and Krishnamurthy, Arvind and Chen, Tianqi and Kasikci, Baris},
  journal={Proceedings of Machine Learning and Systems},
  volume={6},
  pages={196--209},
  year={2024}
}

@article{lin2024awq,
  title={Awq: Activation-aware weight quantization for on-device llm compression and acceleration},
  author={Lin, Ji and Tang, Jiaming and Tang, Haotian and Yang, Shang and Chen, Wei-Ming and Wang, Wei-Chen and Xiao, Guangxuan and Dang, Xingyu and Gan, Chuang and Han, Song},
  journal={Proceedings of machine learning and systems},
  volume={6},
  pages={87--100},
  year={2024}
}

@article{wan2026attention,
  title={Attention Routing Stabilizes Early: Working-Set Inference for Recurrent Language Models},
  author={Wan, Ke and Chen, Chen},
  journal={arXiv preprint arXiv:2609.27373},
  year={2026}
}

@inproceedings{saunshi2025reasoning,
  title={Reasoning with latent thoughts: On the power of looped transformers},
  author={Saunshi, Nikunj and Dikkala, Nishanth and Li, Zhiyuan and Kumar, Sanjiv and J Reddi, Sashank},
  booktitle={International Conference on Learning Representations},
  volume={2025},
  pages={14855--14881},
  year={2025}
}

@article{jeddi2026loopformer,
  title={Loopformer: Elastic-depth looped transformers for latent reasoning via shortcut modulation},
  author={Jeddi, Ahmadreza and Ciccone, Marco and Taati, Babak},
  journal={arXiv preprint arXiv:2602.11451},
  year={2026}
}

@article{dehghani2018universal,
  title={Universal transformers},
  author={Dehghani, Mostafa and Gouws, Stephan and Vinyals, Oriol and Uszkoreit, Jakob and Kaiser, {\L}ukasz},
  journal={arXiv preprint arXiv:1807.03819},
  year={2018}
}

@article{yang2026stabilizing,
  title={Stabilizing recurrent dynamics for test-time scalable latent reasoning in looped language models},
  author={Yang, Xiao-Wen and Han, Ziyu and Zhang, Xi-Hua and Wei, Wen-Da and Shao, Jie-Jing and Guo, Lan-Zhe and Li, Yu-Feng},
  journal={arXiv preprint arXiv:2605.26733},
  year={2026}
}

@inproceedings{xiao2023smoothquant,
  title={Smoothquant: Accurate and efficient post-training quantization for large language models},
  author={Xiao, Guangxuan and Lin, Ji and Seznec, Mickael and Wu, Hao and Demouth, Julien and Han, Song},
  booktitle={International conference on machine learning},
  pages={38087--38099},
  year={2023},
  organization={PMLR}
}
\bibliographystyle{tmlr}

\newpage
\appendix
\appendix

\appendix

\section{Proofs for Quantization Error Analysis}
\label{app:error_analysis_proofs}

\subsection{Proof of Theorem~\ref{thm:error_propagation}}
\label{app:proof_error_propagation}

\begin{proof}
For a quantized recurrent step, $q_t=1$, and the deviation from the reference trajectory is
\begin{equation}
e_{t+1}
=
\left\|
\widetilde{F}_{\theta,t}
\!\left(
\widetilde{\boldsymbol{h}}^{(t)}
\right)
-
F_{\theta}
\!\left(
\boldsymbol{h}^{(t)}
\right)
\right\|.
\label{eq:proof_quantized_start}
\end{equation}
Adding and subtracting $F_{\theta}(\widetilde{\boldsymbol{h}}^{(t)})$ and applying the triangle inequality gives
\begin{equation}
\begin{aligned}
e_{t+1}
&\leq
\left\|
\widetilde{F}_{\theta,t}
\!\left(
\widetilde{\boldsymbol{h}}^{(t)}
\right)
-
F_{\theta}
\!\left(
\widetilde{\boldsymbol{h}}^{(t)}
\right)
\right\| \\
&\quad+
\left\|
F_{\theta}
\!\left(
\widetilde{\boldsymbol{h}}^{(t)}
\right)
-
F_{\theta}
\!\left(
\boldsymbol{h}^{(t)}
\right)
\right\|.
\end{aligned}
\label{eq:proof_triangle}
\end{equation}
Applying Equations~\ref{eq:uniform_quantization_error} and~\ref{eq:local_contraction} yields
\begin{equation}
e_{t+1}
\leq
\delta+\rho_t e_t.
\label{eq:proof_quantized_step}
\end{equation}
For a reference-precision step, $q_t=0$, both trajectories are propagated by $F_{\theta}$, and therefore
\begin{equation}
\begin{aligned}
e_{t+1}
&=
\left\|
F_{\theta}
\!\left(
\widetilde{\boldsymbol{h}}^{(t)}
\right)
-
F_{\theta}
\!\left(
\boldsymbol{h}^{(t)}
\right)
\right\| \\
&\leq
\rho_t e_t.
\end{aligned}
\label{eq:proof_reference_step}
\end{equation}
The two cases can therefore be written jointly as
\begin{equation}
e_{t+1}
\leq
\rho_t e_t
+
q_t\delta.
\label{eq:proof_joint_recurrence}
\end{equation}
Since the two trajectories share the same initial representation, $e_0=0$. Repeatedly applying Equation~\ref{eq:proof_joint_recurrence} gives
\begin{equation}
\begin{aligned}
e_T
&\leq
q_{T-1}\delta
+
\rho_{T-1}q_{T-2}\delta
+
\rho_{T-1}\rho_{T-2}q_{T-3}\delta
+
\cdots \\
&\quad+
q_0\delta
\prod_{j=1}^{T-1}
\rho_j.
\end{aligned}
\label{eq:proof_expanded_schedule}
\end{equation}
Collecting these terms yields
\begin{equation}
e_T(\boldsymbol{q})
\leq
\delta
\sum_{i=0}^{T-1}
q_i
\prod_{j=i+1}^{T-1}
\rho_j
=
E(\boldsymbol{q}),
\label{eq:proof_general_schedule_bound}
\end{equation}
which proves Theorem~\ref{thm:error_propagation}.
\end{proof}

\subsection{Proof of Theorem~\ref{thm:optimal_placement}}
\label{app:proof_optimal_placement}

\begin{proof}
Define the final-error weight associated with quantizing recurrent step $i$ as
\begin{equation}
w_i
=
\delta
\prod_{j=i+1}^{T-1}
\rho_j.
\label{eq:error_weight_definition}
\end{equation}
The bound in Equation~\ref{eq:general_schedule_bound} can then be written as
\begin{equation}
E(\boldsymbol{q})
=
\sum_{i=0}^{T-1}
q_i w_i.
\label{eq:weighted_schedule_bound}
\end{equation}
For two recurrent steps $i<k$, their weights satisfy
\begin{equation}
w_i
=
\left(
\prod_{j=i+1}^{k}
\rho_j
\right)
w_k.
\label{eq:weight_relation}
\end{equation}
Because $0\leq\rho_j<1$, we have
\begin{equation}
w_i
\leq
w_k
\qquad
\text{for all } i<k.
\label{eq:weight_monotonicity}
\end{equation}
Thus, the error weights are nondecreasing with recurrent depth: earlier quantized steps have no larger final-error weight than later quantized steps.

Consider any precision schedule $\boldsymbol{q}$ containing exactly $K$ quantized steps. Its error bound is the sum of exactly $K$ elements selected from the ordered sequence
\begin{equation}
w_0
\leq
w_1
\leq
\cdots
\leq
w_{T-1}.
\label{eq:ordered_error_weights}
\end{equation}
The smallest possible sum of $K$ such elements is obtained by selecting the first $K$ weights, corresponding to
\begin{equation}
q_t
=
1
\quad
\text{for }
t=0,\ldots,K-1,
\label{eq:minimum_selection}
\end{equation}
which is precisely $\boldsymbol{q}^{\mathrm{RAMP}}$. Therefore,
\begin{equation}
E\!\left(\boldsymbol{q}^{\mathrm{RAMP}}\right)
\leq
E(\boldsymbol{q})
\label{eq:ramp_minimum_bound}
\end{equation}
for every schedule with $K$ quantized steps.

Likewise, the largest possible sum of $K$ weights is obtained by selecting the final $K$ elements,
\begin{equation}
q_t
=
1
\quad
\text{for }
t=T-K,\ldots,T-1,
\label{eq:maximum_selection}
\end{equation}
which corresponds to $\boldsymbol{q}^{\mathrm{Rev}}$. Hence,
\begin{equation}
E(\boldsymbol{q})
\leq
E\!\left(\boldsymbol{q}^{\mathrm{Rev}}\right).
\label{eq:reverse_maximum_bound}
\end{equation}
Combining Equations~\ref{eq:ramp_minimum_bound} and~\ref{eq:reverse_maximum_bound} gives
\begin{equation}
E\!\left(\boldsymbol{q}^{\mathrm{RAMP}}\right)
\leq
E(\boldsymbol{q})
\leq
E\!\left(\boldsymbol{q}^{\mathrm{Rev}}\right),
\label{eq:proof_schedule_ordering}
\end{equation}
which proves Theorem~\ref{thm:optimal_placement}. If the relevant contraction factors satisfy $0<\rho_j<1$, the ordering of the corresponding weights is strict whenever an exchanged pair spans at least one strictly contractive recurrent step, so any non-prefix schedule has a strictly larger bound than the corresponding \texttt{RAMP} schedule unless the exchanged weights are equal.
\end{proof}

\newpage

\subsection{Full Long-Context Fidelity Results}
\label{app:context_scaling}

\paragraph{WikiText context scaling.}
Table~\ref{tab:wikitext_context_scaling} reports the complete long-context
fidelity results underlying the WikiText quality curve in
Figure~\ref{fig:ramp_context_scaling}. We use the same 32 frozen WikiText
endpoints at every context length, with nested contexts ending at the same
corpus location and identical final 64 prediction targets. All metrics are
computed under the same tail-64 evaluation protocol. Across context lengths
from 128 to 4096 tokens, \texttt{RAMP} remains consistently close to the
BF16-32 trajectory: paired $\Delta$NLL remains near zero, while hidden-state
and logit errors remain stable.

\begin{table}[t]
    \centering
    \caption{
    \textbf{WikiText long-context fidelity of \texttt{RAMP}.}
    Results use 32 frozen evaluation endpoints with nested contexts and
    identical tail-64 prediction targets across context lengths.
    $\Delta$NLL is measured relative to BF16-32.
    Hidden-state and logit relative $\ell_2$ errors are reported as percentages.
    }
    \label{tab:wikitext_context_scaling}
    \small
    \setlength{\tabcolsep}{7pt}
    \renewcommand{\arraystretch}{1.08}
    \begin{tabular}{@{}rccccc@{}}
        \toprule
        \textbf{Context}
        & \textbf{BF16 NLL}
        & \textbf{\texttt{RAMP} NLL}
        & $\Delta$\textbf{NLL}
        & \textbf{Hidden} $\ell_2$
        & \textbf{Logit} $\ell_2$ \\
        \midrule

        128
        & 2.603711
        & 2.605315
        & +0.001604
        & 2.937\%
        & 1.429\% \\

        256
        & 2.345469
        & 2.347890
        & +0.002422
        & 2.881\%
        & 1.647\% \\

        512
        & 2.216275
        & 2.217365
        & +0.001090
        & 2.699\%
        & 1.378\% \\

        1024
        & 2.153023
        & 2.153203
        & +0.000179
        & 2.556\%
        & 1.277\% \\

        2048
        & 2.116153
        & 2.115046
        & $-0.001107$
        & 2.586\%
        & 1.294\% \\

        4096
        & 2.101984
        & 2.102086
        & +0.000102
        & 2.653\%
        & 1.342\% \\

        \bottomrule
    \end{tabular}
\end{table}

\paragraph{PG19 context scaling.}
We repeat the same long-context fidelity analysis on 32 frozen PG19 validation
books. Contexts are constructed entirely within individual documents, and the
final 64 prediction targets remain identical for each book as context length
changes. As shown in Table~\ref{tab:pg19_context_scaling}, \texttt{RAMP} again
remains close to BF16-32 throughout the evaluated range. Hidden-state and
logit errors remain stable from 512 to 4096 tokens, while paired $\Delta$NLL
stays small. PG19 exhibits a mild positive trend in $\Delta$NLL toward longer
contexts; the negative values at shorter contexts should therefore be
interpreted as small numerical deviations rather than evidence that
\texttt{RAMP} improves over BF16.

\begin{table}[t]
    \centering
    \caption{
    \textbf{PG19 long-context fidelity of \texttt{RAMP}.}
    Results use 32 frozen validation books with document-contained nested
    contexts and identical tail-64 prediction targets across context lengths.
    $\Delta$NLL is measured relative to BF16-32.
    Hidden-state and logit relative $\ell_2$ errors are reported as percentages.
    }
    \label{tab:pg19_context_scaling}
    \small
    \setlength{\tabcolsep}{7pt}
    \renewcommand{\arraystretch}{1.08}
    \begin{tabular}{@{}rccccc@{}}
        \toprule
        \textbf{Context}
        & \textbf{BF16 NLL}
        & \textbf{\texttt{RAMP} NLL}
        & $\Delta$\textbf{NLL}
        & \textbf{Hidden} $\ell_2$
        & \textbf{Logit} $\ell_2$ \\
        \midrule

        512
        & 2.671946
        & 2.669488
        & $-0.002457$
        & 2.722\%
        & 1.201\% \\

        1024
        & 2.623134
        & 2.621616
        & $-0.001518$
        & 2.825\%
        & 1.325\% \\

        2048
        & 2.607391
        & 2.607932
        & +0.000542
        & 2.792\%
        & 1.234\% \\

        4096
        & 2.592401
        & 2.592470
        & +0.000068
        & 3.075\%
        & 1.354\% \\

        \bottomrule
    \end{tabular}
\end{table}

\paragraph{Cross-dataset long-context behavior.}
The two datasets exhibit a consistent long-context fidelity pattern.
On WikiText, paired $\Delta$NLL remains near zero throughout
128--4096 tokens, while hidden-state and logit errors remain approximately
constant. PG19 shows the same stability in representation- and output-level
error over 512--4096 tokens, with paired $\Delta$NLL remaining small despite
a mild positive trend toward longer contexts. Together with the 4K
cross-dataset comparison in Table~\ref{tab:quality_4k}, these results show
that terminal BF16 refinement preserves near-full-precision recurrent
trajectories across both datasets and remains effective as context length
increases.

\end{document}